\documentclass[12pt]{article}

\usepackage{graphicx} 
\usepackage{float}    
\usepackage{verbatim} 
\usepackage{amsmath}  
\usepackage{amssymb} 
\usepackage{mathtools} 
\usepackage{subfig}   
\usepackage{times}
\usepackage[colorlinks=true,citecolor=blue,linkcolor=blue,urlcolor=blue]{hyperref}
\usepackage{fullpage}
\usepackage{amsthm}
\usepackage{enumerate}
\usepackage{paralist}
\usepackage{xspace}
\usepackage{caption}
\usepackage{bbm}
\usepackage{url}
\usepackage[square, numbers]{natbib}
\usepackage{mathrsfs}
\usepackage{algorithm,algcompatible}
\usepackage{thmtools}
\usepackage{thm-restate}
\usepackage{dsfont}

\newcounter{example}
\newenvironment{example}[1][]{\refstepcounter{example}\par\medskip
\noindent \textbf{Example~\theexample. #1} \rmfamily}{\medskip}

\newcounter{assumption}
\newenvironment{assumption}[1][]{\refstepcounter{assumption}\par\medskip
\noindent \textbf{Assumption~\theassumption. #1} \rmfamily}{\medskip}

\newcommand\xqed[1]{%
  \leavevmode\unskip\penalty9999 \hbox{}\nobreak\hfill
  \quad\hbox{#1}}
\newcommand\exend{\xqed{$\square$}}

\newtheorem{prop}{Proposition}
\newtheorem{lemma}{Lemma}

\newtheorem{thm}{Theorem}
\newtheorem*{thm*}{Theorem}

\DeclareMathOperator*{\sgn}{sgn}

\DeclareMathOperator*{\diag}{diag}
\DeclareMathOperator*{\tr}{\mathbf{tr}}

\newcommand{\Vcal}{\ensuremath{\mathcal{V}}}
\newcommand{\Sb}{\ensuremath{\mathbb{S}}}

\newcommand{\Dcal}{\ensuremath{\mathcal{D}}}
\newcommand{\Ocal}{\ensuremath{\mathcal{O}}}

\newcommand{\norm}[1]{\lVert #1 \rVert}

\newcommand{\EE}{\mathbb{E}}
\newcommand{\BB}{\mathbb{B}}

\newcommand{\PP}{\mathbb{P}}

\newcommand{\RR}{\mathbb{R}}

\newcommand{\Ahat}{\widehat{A}}
\newcommand{\Astar}{A_\star}

\newcommand{\vb}[1]{\mathbf{#1}} 

\newcommand{\indi}{\mathds{1}}

\newcommand{\ctcontrol}{c_1}
\newcommand{\ctinitial}{c_2}
\newcommand{\ctols}{c_3}
\newcommand{\ctbeta}{c_4}

\title{Active Learning for Nonlinear System Identification with Guarantees}
\author{Horia Mania \qquad Michael I. Jordan \qquad  Benjamin Recht \\
Department of Electrical Engineering and Computer Science \\
University of California, Berkeley}
\date{June 18, 2020}

\begin{document}
\maketitle

\begin{abstract}
While the identification of nonlinear dynamical systems is a fundamental building block of model-based reinforcement learning and feedback control, its sample complexity is only understood for systems that either have discrete states and actions or for systems that can be identified from data generated by i.i.d. random inputs. 
Nonetheless, many interesting dynamical systems have continuous states and actions and can only be identified through a judicious choice of inputs. 
Motivated by practical settings, we study a class of nonlinear dynamical systems whose state transitions depend linearly on a known feature embedding of state-action pairs. To estimate such systems in finite time identification methods must explore all directions in feature space. 
We propose an active learning approach that achieves this by repeating three steps: trajectory planning, trajectory tracking, and re-estimation of the system from all available data. 
We show that our method estimates nonlinear dynamical systems at a parametric rate, similar to the statistical rate of standard linear regression. 
\end{abstract}

\section{Introduction}

The estimation of nonlinear dynamical systems with continuous states and inputs is generally based on data collection procedures inspired by the study of optimal input design for linear dynamical systems \cite{schoukens2019nonlinear}. 
Unfortunately, these data collection methods are not guaranteed to enable the estimation of nonlinear systems. To resolve this issue, studies of system identification either assume the available data is informative enough for estimation  \cite{hong2008model, ljung1987system, schoukens2019nonlinear} or consider systems for which i.i.d. random inputs produce informative data \cite{bahmani2019convex, foster2020learning, oymak2018stochastic,sattar2020non}. 
However, as we will see, there are many nonlinear dynamical systems that cannot be estimated without a judicious choice of inputs.

Inspired by experimental design and active learning, we present a data collection scheme that is guaranteed to enable system identification in finite time. Our method applies to dynamical systems whose transitions depend linearly on a known feature embedding of state-input pairs. This class of models can capture many types of systems and is used widely in system identification \cite{hong2008model,ljung1987system}. For example, \citet{ng2006autonomous} used such a model to estimate the dynamics of a helicopter and \citet{brunton2016discovering} showed that sparse linear regression of polynomial and trigonometric feature embeddings can be used to fit models of the chaotic Lorentz system and of a fluid shedding behind an obstacle. These models can be parametrized as follows: 
\begin{align}
\label{eq:system}
\vb x_{t + 1} = A_\star \phi(\vb x_t, \vb u_t) + \vb w_t,
\end{align}
where $\vb x_t$ and $\vb u_t$ are the state and input of the system at time $t$, and $\vb w_t$ is stochastic noise. The feature map $\phi$ is assumed known and the goal is to estimate $\Astar$ from one trajectory by choosing a good sequence of inputs. The input $\vb u_t$ is allowed to depend on the history of states $\{\vb x_j\}_{j  = 0}^t$ and is independent of $\vb w_t$. 

The class of systems \eqref{eq:system} contains any linear system, with fully observed states, when the features include the states and inputs of the system. Moreover, any piecewise affine (PWA) system can be expressed using \eqref{eq:system} if the support of its pieces is known. First introduced by \citet{sontag1981nonlinear} as an approximation of nonlinear systems, PWA systems are a popular model of hybrid systems \cite{borrelli2017predictive,camacho2010model,heemels2001equivalence} and have been successfully used in a wide range of applications \cite{borrelli2006mpc,geyer2008hybrid,han2017feedback,marcucci2017approximate,sadraddini2019sampling,yordanov2011temporal}. 

While linear dynamical systems can be estimated from one trajectory produced by i.i.d. random inputs \cite{simchowitz2018learning}, the following example shows that this is not possible for PWA systems. 

\begin{example}
\label{ex:intro}
Let us consider the feature map $\phi : \RR^d \times \RR^d \to \RR^{3d}$ defined by: 
\begin{align*}
\phi(\vb x, \vb u) = \begin{bmatrix}
\vb x \cdot \indi\{\norm{\vb x} \leq \frac{3}{2}\} \\
\vb x \cdot \indi\{\norm{\vb x} > \frac{3}{2}\}\\
\vb u \cdot \indi\{\norm{\vb u} \leq 1\}\}
\end{bmatrix},
\end{align*}
where $\indi\{\cdot\}$ is the indicator function and the multiplication with $\vb x$ is coordinatewise. 
We assume there is no process noise and let $\Astar = \begin{bmatrix} \frac{1}{2} I_d & A_2 & I_d \end{bmatrix}$ for some $d \times d$ matrix $A_2$ and the $d \times d$ identity matrix $I_d$. Also, we assume $\vb x_0 = 0$.

Then, since the inputs to the system can have magnitude at most $1$, the state of the system can have magnitude larger than $3/2$ only if consecutive inputs point in the same direction. However, the probability that two or more random vectors, uniformly distributed on the unit sphere, point in the same direction is exponentially small in the dimension $d$. Therefore, if we used random inputs, we would have to wait for a long time in order to reach a state with magnitude larger than $3/2$.

On the other hand, if we chose a sequence of inputs $\vb u_t = \vb u$ for a fixed unit vector $\vb u$, we would be guaranteed to reach a state with norm larger than $3/2$ in a couple of steps. Hence, despite the input constraint, we would be able to reach the region $\norm{\vb x} > 3/2$ with a good choice of inputs. \exend
\end{example}

Therefore, the estimation of \eqref{eq:system} requires a judicious choice of inputs. To address this challenge we propose a method based on trajectory planning. At a high level, our method repeats the following three steps: 
\begin{itemize}
\item Given past observations and an estimate $\Ahat$, our method plans a reference trajectory from the current state of the system to a high uncertainty region of the feature space.

\item Then, our method attempts to track the reference trajectory using $\Ahat$. 

\item Finally, using all data collected so far, our method re-estimates $\Ahat$. 
\end{itemize}

The ability to find reference trajectories from a given state to a desired goal set is related to the notion of controllability, a standard notion in control theory. A system is called \emph{controllable} if it is possible to take the system from any state to any other state in a finite number of steps by using an appropriate sequence of inputs. In our case, a system is considered more controllable the bigger we can make the inner product between the system's features and goal directions in feature space. The number of time steps required to obtain a large inner product is called \emph{planning horizon}. 

The controllability of the system and the planning horizon are system dependent properties that influence our ability to estimate the system. Intuitively, the more controllable a system is, the easier it is to collect the data we need to estimate it. The following informal version of our main result clarifies this relationship.

\begin{thm*}[Informal]
Our method chooses actions $\vb u_t$ such that with high probability the ordinary least squares (OLS) estimate $\Ahat \in \arg \min_{A} \sum_{t = 0}^{T - 1} \norm{A \phi (\vb x_t, \vb u_t) - \vb x_{t + 1}}^2$ satisfies 
\begin{align*}
\norm{\Ahat - \Astar} \leq \frac{\text{size of the noise}}{\text{controllability of the system}} \sqrt{\frac{\text{dimension} \times \text{planning horizon}}{\text{number of data points}}}.
\end{align*}
\end{thm*}


This statistical rate is akin to that of standard supervised linear regression, but it has an additional dependence on the controllability of the system and the planning horizon. To better understand why these two terms appear, recall that our method uses $\Ahat$, an estimate of $\Astar$, to plan and track reference trajectories. Therefore, the tracking step is not guaranteed to reach the desired region of the feature space. The main insight of our analysis is that when trajectory tracking fails, we are still guaranteed to collect at least one informative data point per reference trajectory. Therefore, in the worst case, the effective size of the data collected by our method is equal to the total number of data points collected over the planning horizon. 




In the next section we present our mathematical assumptions and in Section~\ref{sec:main} we discuss our method and main result. Section~\ref{sec:ols} includes a general result about linear regression of dependent data derived from prior work. Then, in Section~\ref{sec:proof} we present in detail the proof of our main result. 
There is a long line of work studying system identification, which we discuss in Section~\ref{sec:related}. Finally, Section~\ref{sec:takeaways} contains takeaways and open problems. 

\textbf{Notation:} The norm $\norm{\cdot}$ is the Euclidean norm whenever it is applied to vectors and is the spectral norm whenever it is applied to matrices. We use $c_1$, $c_2$, $c_3$, \ldots to denote different universal constants. Also, $\Sb^{p - 1}$ is the unit sphere in $\RR^p$ and $\BB^p_r$ is the ball in $\RR^p$ centered at the origin and of radius $r$. 
The symbol $\square$ is used to indicate the end of an example or of a proof. 

\section{Assumptions}

To guarantee the estimation of \eqref{eq:system} we must make several assumptions about the true system we are trying to identify. We denote the dimensions of the states and inputs by $d$ and $p$ respectively. The feature map $\phi$ maps state-action pairs to feature vectors in $\RR^k$. 

The main challenge in the estimation of \eqref{eq:system} is choosing inputs $\vb u_t$ so that the minimal singular value of the design matrix is $\Omega(\sqrt{T})$, where $T$ is the length of the trajectory collected from the system. To reliably achieve this we must assume the feature map $\phi$ has some degree of smoothness. Without a smoothness assumption the noise term $\vb w_t$ at time $t$ might affect the feature vector $\phi(\vb x_{t + 1}, \vb u_{t + 1})$ at time $t + 1$ in arbitrary ways, regardless of the choice of input at time $t$. 

\begin{assumption}
\label{as:lip}
The map $\phi \colon \RR^d \times \BB^p_{r_u} \to \RR^k$ is $L$-Lipschitz\footnote{Since $\phi$ is continuous and since $\vb u$ lies in a compact set, we know that any continuous function of $\phi(\vb x, \vb u)$ achieves its maximum and minimum with respect to $\vb u$. This is the only reason we assume the inputs to the system are bounded. Alternatively, we could let the inputs be unbounded and work with approximate maximizers and minimizers.}. 
\end{assumption}

In order to use known techniques for the analysis of online linear least squares \cite{abbasi2011online, dani2008stochastic,rusmevichientong2010linearly, simchowitz2018learning} we also assume that the feature map $\phi$ is bounded. For some classes of systems (e.g. certain linear systems) this condition can be removed \cite{simchowitz2018learning}. 

\begin{assumption}
\label{as:bounded}
There exists $b_\phi > 0$ such that $\norm{\phi(\vb x, \vb u)} \leq b_\phi$ for all $\vb x \in \RR^d$ and $\vb u \in \BB_{r_u}$.
\end{assumption}

This assumption implies that the states of the system \eqref{eq:system} are bounded, a consequence which can be limiting in some applications. To address this issue we could work with the system 
\begin{align}
\label{eq:delta_system}
\vb x_{t + 1} = A_\star \phi(\vb x_t, \vb u_t) + \vb x_t + \vb w_t
\end{align}
instead. In this case, $\phi$ being bounded implies that the increments $\vb x_{t + 1} - \vb x_t$ are bounded, allowing the states to grow in magnitude. However, formulation \eqref{eq:delta_system} complicates the exposition so we choose to focus on \eqref{eq:system}.

As mentioned in the introduction, our method relies on trajectory planning and tracking to determine the inputs to the system. Suppose we would like to track a reference trajectory $\{(\vb x_t^R, \vb u_t^R)\}_{t \geq 0}$ that satisfies $\vb x_{t + 1}^R = \Astar \phi(\vb x_t^R, \vb u_t^R)$. In other words, we wish to choose inputs $\vb u_t$ to ensure that the tracking error $\| \vb x_t - \vb x_t^R\|$ is small. Simply choosing $\vb u_t = \vb u_t^R$ does not work even when the initial states $\vb x_0$ and $\vb x_0^R$ are equal because the true system \eqref{eq:system} experiences process noise.

To ensure that tracking is possible we assume that there always exists an input to the true system that can keep the tracking error small. 
There are multiple ways to formalize such an assumption. We make the following choice. 

\begin{assumption}
\label{as:control}
There exist positive constants $\gamma$ and $b_u$ such that for any $\vb x, \vb x^\prime \in \RR^d$ and any $\vb u^\prime \in \BB^p_{b_u}$ we have
\begin{align}
\label{eq:control_as}
\min_{\vb u \in \BB^p_{r_u}}\norm{ A_\star \left(\phi(\vb x, \vb u) - \phi(\vb x^\prime, \vb u^\prime) \right) }\leq \gamma \norm{ \vb x - \vb x^\prime}. 
\end{align}
Moreover, if $\norm{\vb u^\prime} \leq b_u / 2$, there exists $\vb u$, with $\norm{\vb u} \leq b_u$, that satisfies \eqref{eq:control_as}. 
\end{assumption}

Therefore, if we wish to track a reference trajectory $\{(\vb x_t^R, \vb u_t^R)\}_{t \geq 0}$ that satisfies $\vb x_{t + 1} = \Astar \phi(\vb x_t, \vb u_t)$, Assumption~\ref{as:control} guarantees the existence of an input $\vb u_t \in \BB^p_{b_u}$ such that 
\begin{align*}
\norm{\vb x_{t + 1} - \vb x_{t + 1}^R} &= \norm{\Astar \phi(\vb x_t, \vb u_t) + \vb w_t - A_\star \phi(\vb x_t^R, \vb u_t^R)} \\
&\leq \gamma \norm{\vb x_t - \vb x_t^R} + \norm{\vb w_t}. 
\end{align*}
In other words, Assumption~\ref{as:control} allows us to find an input $\vb u_t$ such that the tracking error $\norm{\vb x_{t+1} - \vb x_{t + 1}^R}$ is upper bounded in terms of noise $\vb w_t$ and the tracking error at time $t$. By induction, Assumption~\ref{as:control} guarantees the existence of inputs to the system such that 
\begin{align*}
\norm{\vb x_H - \vb x_H^R} \leq \max_{t = 0, \ldots, H - 1} \norm{\vb w_t} (1 + \gamma + \ldots + \gamma^{H - 1}) + \gamma^H \norm{\vb x_0 - \vb x_0^R}.
\end{align*}
Hence, when $\gamma < 1$ we can choose a sequence of inputs such that the state $\vb x_H$ at time $H$ is close to $\vb x_{H}^R$, as long as the process noise is well behaved.  

Note that in Assumption~\ref{as:control} we allow $\gamma \geq 1$. However, we pay a price when $\gamma$ is large. The larger $\gamma$ is the more stringent the next assumptions become. Finally, we note that the parameter $b_u$ appearing in Assumption~\ref{as:control} makes it easier for systems to satisfy the assumption than requiring that \eqref{eq:control_as} holds for all $\vb u^\prime$.

To estimate  \eqref{eq:system} we must collect measurements of state transitions from feature vectors that point in different directions. To ensure that such data can be collected from the system we must assume that there exist sequences of actions which take the dynamical system from a given state to some desired direction in feature space. This type of assumption is akin to the notion of controllability, which is standard in control theory. For example, a linear system $\vb x_{t + 1} = A \vb x_t + B \vb u_t$ is said to be controllable when the matrix $\begin{bmatrix} B & AB & \ldots & A^{d - 1} B \end{bmatrix}$ has full row rank. The interested reader can easily check that for a controllable linear system it is possible to get from any state to any other state in $d$ steps by appropriately choosing a sequence of inputs. 

This notion of controllability can be extended to a class of nonlinear systems, called control affine systems, through the use of Lie brackets \cite{sastry2013nonlinear,slotine1991applied}. We require a different notion of controllability. Namely, we assume that in the absence of process noise we can take the system \eqref{eq:system} from any state to a feature vector that aligns sufficiently with a desired direction in feature space. 

\begin{assumption}
\label{as:controllable}
There exist $\alpha$ and $H$, a positive real number and a positive integer, such that for any initial state $\vb x_0$ and goal vector $v \in \Sb^{k - 1}$ there exists a sequence of actions $\vb u_t$, with $\norm{\vb u_t} \leq b_u / 2$, such that $\left| \langle \phi(\vb x_t, \vb u_t), v \rangle \right| \geq \alpha > 0$ for some $0 \leq t \leq H$, with $\vb x_{j + 1} = A_\star \phi (\vb x_j, \vb u_j)$ for all $j$. 
\end{assumption}

If the assumption is satisfied for some horizon $H$, it is clear that it is also satisfied for larger horizons. Moreover, one expects that a larger horizon $H$ allows a larger controllability parameter $\alpha$. 
As discussed in the introduction, the larger $H$ is, the weaker our guarantee on estimation will be. However, the larger $\alpha$ is, the better our guarantee on estimation will be. Therefore, there is a tension between $\alpha$ and $H$ in our final result. 

Assumptions~\ref{as:lip} to \ref{as:controllable} impose many constraints. Therefore, it is important to give examples of nonlinear dynamical systems that satisfy these assumptions. We give two simple examples. First we present a synthetic example for which it is easy to check that it satisfies all the assumptions, and then we discuss the simple pendulum. 

\begin{example}[Smoothed Piecewise Linear System]
When the support sets of the different pieces are known, piecewise affine systems can be easily expressed as ~\eqref{eq:system}. However, the feature map $\phi$ would not be continuous. In this example, we present a smoothed version of a PWA system, which admits a $1$-Lipschitz feature map. Let $f : \RR \to \RR$ be defined by
\begin{align}
f(x) = \left\{ 
\begin{array}{ll}
0 & \text{if } x < -1/2,\\
x + 1/2 & \text{if } x \in [-1/2, 1/2],\\
1 & \text{if } x > 1/2.
\end{array}
\right.
\end{align}
We also consider the maps $g(\vb x) = \frac{\vb x_{t}}{\norm{\vb x_t}} \min \{\norm{\vb x_t}, b_x\}$ and $h(\vb u) = \frac{\vb u}{\norm{\vb u}} \min\{\norm{\vb u}, r_u\}$, for some values $b_x$ and $r_u$. In this example both the inputs and the states are $d$ dimensional. Then, we define the feature map $\phi : \RR^{2d} \to \RR^{3d}$ as follows
\begin{align}
\phi(\vb x, \vb u) = \begin{bmatrix}
g(\vb x) f( x_{1}) \\
g(\vb x)(1 - f(x_{1})) \\
h(\vb u)
\end{bmatrix},
\end{align}
where $x_{1}$ denotes the first coordinate of $\vb x$. Now, let us consider the following dynamical system:
\begin{align}
\label{eq:spls}
\vb x_{t + 1} = \begin{bmatrix}A_1 & A_2 & I_d \end{bmatrix}\phi(\vb x_t, \vb u_t) +   \vb w_t,
\end{align}
where $A_1$ and $A_2$ are two unknown $d \times d$ matrices. For the purpose of this example we can assume the noise $\vb w_t$ is zero almost surely. 

To better understand the system \eqref{eq:spls} note that when $\norm{\vb x_t} \leq b_x$ and $\norm{\vb u_t} \leq r_u$ we have
\begin{align*}
\vb x_{t + 1} &= A_1 \vb x_{t} + \vb u_t \quad \text{if }\; x_{t1} \geq 1/2, \\
\vb x_{t + 1} &= A_2 \vb x_{t} + \vb u_t \quad \text{if }\; x_{t1} \leq -1/2.
\end{align*}
By construction, the feature map of the system is $1$-Lipschitz and bounded. Therefore, \eqref{eq:spls} satisfies Assumptions \ref{as:lip}, and \ref{as:bounded}. We are left to show that we can choose $A_1$, $A_2$, $b_x$, and $r_u$ so that \eqref{eq:spls} satisfies Assumptions \ref{as:control} and \ref{as:controllable} as well. 

It is easy to convince oneself that if $b_x > 2\sqrt{2}$, Assumptions~\ref{as:control} and \ref{as:controllable} hold for any $A_1$ and $A_2$ as long as $r_u$ is sufficiently large relative to $A_1$, $A_2$, and $b_x$. In fact, if $r_u$ is sufficiently large, Assumption~\ref{as:control} is satisfied with $\gamma = 0$. \exend
\end{example}

\begin{example}[Simple Pendulum]
We know that the dynamics of a simple pendulum in continuous time are described by the equation 
\begin{align}
\label{eq:pendulum}
m\ell^2 \ddot{\theta}(t) + mg\ell \sin \theta(t)  = - b\dot{\theta}(t) + u(t),
\end{align}
where $\theta(t)$ is the angle of the pendulum at time $t$, $m$ is the mass of the pendulum, $\ell$ is its length, $b$ is a friction coefficient, and $g$ is the gravitational acceleration. 

Then, if we discretize \eqref{eq:pendulum} according to Euler's method~\footnote{Using a more refined discretization method, such as a Runge-Kutta method, would be more appropriate. Sadly, such discretization methods yield a discrete time system which cannot be easily put in the form \eqref{eq:delta_system}.} with step size $h$ and assume stochastic process noise, we obtain the two dimensional system:
\begin{align}
\vb x_{t + 1} = \vb x_t + \begin{bmatrix}
a_1 & a_2\\
h & 0
\end{bmatrix} \begin{bmatrix}
x_{t1} \\
\sin \left( x_{t2}\right)
\end{bmatrix} + \begin{bmatrix}
a_3 \\
0
\end{bmatrix} \vb u_t + \vb w_t, 
\end{align}
where $x_{t1}$ and $x_{t2}$ are the  coordinates of $\vb x_t$ and $a_1$, $a_2$, and $a_3$ are unknown real values. The first coordinate of $\vb x_t$ represents the angular velocity of the pendulum at time $t$, while the second coordinate represents the angle of the pendulum. Therefore, to put the inverted pendulum in the form of \eqref{eq:delta_system} we can consider the feature map 
\begin{align}
\phi(\vb x_t, \vb u_t) = \begin{bmatrix}
 x_{t1} \\
\sin \left( x_{t2}\right) \\
\vb u_t
\end{bmatrix}.
\end{align}
It can be easily checked that this feature map is $1$-Lipschitz. While it is not bounded, if pendulum experiences friction, we can ensure the feature values stay bounded by clipping the inputs $\vb u_t$, i.e. replace $\vb u_t$ with $\sgn(\vb u_t) \min \{|\vb u_t|, r_u\}$ for some value $r_u$. 

The simple pendulum satisfies Assumption~\ref{as:controllable} because we can drive the system in a finite number of steps from any state $\vb x_t$ to states $\vb x_{t + H}$ for which the signs of $ x_{(t + h)1}$
and $\sin \left( x_{t1}\right)$ can take any value in $\{-1, 1\}^2$, with their absolute values lower bounded away from zero. 

Finally, Assumption~\ref{as:control} holds with $\gamma \geq 1 + h$. This assumption is pessimistic because the simple pendulum is stabilizable and can track reference trajectories. However, Assumption~\ref{as:control} does not hold with $\gamma < 1$ since the input at time $t$ does not affect the position at time $t + 1$. \exend
\end{example}

Now we turn to our last two assumptions. We need to make an assumption about the process noise and we also must assume access to an initial $\Ahat$ to warm start our method. 

\begin{assumption}
\label{as:noise}
The random vectors $\vb w_t$ are independent, zero mean,  and $\norm{\vb w_t} \leq b_w$ a.s.\footnote{We can relax this assumption to only require $\vb w_t$ to be sub-Gaussian. In this case, we would make a truncation argument to obtain an upper bound on all $\vb w_t$ with high probability.}. Also, $\vb w_t$ is independent of $(\vb x_t, \vb u_t)$. Furthermore, we assume 
\begin{align}
\label{eq:noise_bound}
b_w \leq \frac{ \alpha}{\ctcontrol L (1 + \gamma + \ldots + \gamma^{H - 1})},
\end{align}%
for some universal constant $\ctcontrol > 2$.   
\end{assumption}

Equation~\ref{eq:noise_bound} imposes on upper bound on the size of the process noise in terms of system dependent quantities: the controllability parameter $\alpha$ introduced in Assumption~\ref{as:controllable}, the Lipschitz constant $L$ of the feature map, and the control parameter $\gamma$ introduced in Assumption~\ref{as:control}. An upper bound on $b_w$ is required because when the process noise is too large, it can be difficult to counteract its effects through feedback.

Finally, we assume access to an initial guess $\Ahat$ with $\norm{\Ahat -\Astar} = \Ocal \left(L^{-1}(1 + \gamma  +\ldots + \gamma^{H - 1})^{-1}\right)$. To understand the key issue this assumption resolves, suppose we are trying to track a reference trajectory $\{(\vb x_t^R, \vb u_t^R)\}_{t \geq 0}$ and $\norm{(\Ahat - \Astar) \phi(\vb x_t^R, \vb u_t^R)}$ is large.
 Without an assumption on the size of $\norm{\Ahat - \Astar}$, the magnitude of $(\Ahat - \Astar) \phi(\vb x_t^R, \vb u_t^R)$ might be large while $\norm{\phi(\vb x_t^R, \vb u_t^R)}$ is small. Then, making a measurement at a point $(\vb x_t, \vb u_t)$ close to $(\vb x_t^R, \vb u_t^R)$ might not be helpful for estimation because $\phi(\vb x_t, \vb u_t)$ could be zero. Therefore, if $\norm{\Ahat - \Astar}$ is too large, we might both fail to track a reference trajectory and to collect a useful measurement. For ease of exposition, instead of assuming access to an initial guess $\Ahat$, we assume access to a dataset. 

\begin{assumption}
\label{as:initial}
We have access to an initial trajectory $\Dcal = \{(\vb x_t, \vb u_t, \vb x_{t + 1})\}_{0 \leq t < t_0}$ of transitions from the true system such that 
\begin{align}
\lambda_{\min} \left(\sum_{t = 0}^{t_0 - 1} \phi(\vb x_t, \vb u_t) \phi(\vb x_t, \vb u_t)^\top \right) \geq 1 + c_2 b_w^2 L^2 \left(\sum_{i = 0}^{H - 1} \gamma^i\right)^2 \left(d + k \log(b_\phi^2 T) + \log\left(\frac{\pi^2 T^2 }{6\delta}\right)\right). 
\end{align}
where $c_2$ is a sufficiently large universal constant and $T$ is the number of samples to be collected by our method. We make explicit the requirement on $c_2$ in Section~\ref{sec:proof}. In Appendix~\ref{app:refine} we show how to replace $T$ by a fixed quantity $T_\star$. 
\end{assumption}

As shown in Section~\ref{sec:ols}, Assumption~\ref{as:initial} guarantees that the OLS estimate $\Ahat$ obtained from $\Dcal$ satisfies $\norm{\Ahat - \Astar} \leq \frac{c_3}{\sqrt{c_2}} L^{-1}(1 + \gamma  +\ldots + \gamma^{H - 1})^{-1}$ for some universal constant $c_3$. Since the features $\phi(\vb x, \vb u)$ can have magnitude as large as $b_\phi$, Assumption~\ref{as:initial} only implies $\norm{(\Ahat - \Astar) \phi(\vb x, \vb u)} = \Ocal(b_\phi L^{-1} (1 + \gamma  +\ldots + \gamma^{H - 1})^{-1})$. Therefore, Assumption~\ref{as:initial} does not imply a stringent upper bound on $\norm{(\Ahat - \Astar) \phi(\vb x, \vb u)}$ because $b_\phi$ can be arbitrarily large relative to $L$ and $\gamma$. 

\section{Main Result}
\label{sec:main}

Our method for estimating the parameters of a dynamical system \eqref{eq:system} is shown in Algorithm~\ref{alg:traj_opt_exploration}. The trajectory planning and tracking  routines are discussed in detail in Sections \ref{sec:plan} and \ref{sec:track} respectively. Our method is also presented in one block of pseudo-code in Appendix~\ref{app:alg}. Now, we can state our main result.

\begin{center}
  \begin{algorithm}[h!]
    \caption{Active learning for nonlinear system identification}{}
    \begin{algorithmic}[1]
      \REQUIRE{Parameters: the feature map $\phi$, initial trajectory $\Dcal$, and parameters $T$, $\alpha$, and $\beta$.}
      \STATE Initialize $\Phi$ to have rows $\phi(\vb x_j, \vb u_j)^\top$ and $Y$ to have rows $(\vb x_{j + 1})^\top$, for $(\vb x_j, \vb u_j, \vb x_{j + 1}) \in \Dcal$.  
      \STATE Set $\widehat{A}  \gets Y^\top \Phi (\Phi^\top \Phi)^{-1}$, i.e. the OLS estimate according to $\Dcal$.  
      \STATE Set $t \gets t_0$.
      \WHILE{$t \leq T + t_0$}
      \STATE Set $\vb x_0^R \gets \vb x_t$,
      \STATE Set $v$ to be a minimal eigenvector of $\Phi^\top \Phi$, with $\norm{v} = 1$.

      \STATE \textbf{Trajectory planning:} find inputs $\vb u_0^R$, $\vb u_1^R$, \ldots, $\vb u_r^R$, with $\norm{\vb u_j^R} \leq b_u$ and $r \leq H$, such that 
      \begin{align*}
&\left | \langle \phi(\vb x^R_r, \vb u_r^R) , v \rangle \right | \geq \frac{\alpha}{2} \text{ or }  \phi(\vb{x}^R_r, \vb u_r^R)^\top (\Phi^\top \Phi)^{-1} \phi(\vb{x}^R_r, \vb u_r^R) \geq \beta,
     \end{align*}
     where $\vb x_{j + 1}^R = \Ahat \phi(\vb x_j^R, \vb u_j^R)$ for all $j \in \{0, 1, \ldots, r - 1\}$.  
     \label{line:planning}

     \STATE \textbf{Trajectory tracking:} track the reference trajectory  $\{(\vb x_j^R, \vb u_j^R)\}_{j = 0}^r$ and increment $t$ as described in Section~\ref{sec:track}.

      \STATE  Set $\Phi^\top \gets [\phi_0, \phi_1, \ldots, \phi_{t -1}]$ and $Y^\top \gets [\vb x_1, \vb x_2, \ldots, \vb x_{t}]$, where $(\phi_j, \vb x_{j + 1})$ are all feature-state transitions observed so far. 
      \STATE \textbf{Re-estimate:} $\widehat{A} \gets Y^\top \Phi (\Phi^\top \Phi)^{-1}$.
      \ENDWHILE
    \STATE Output the last estimates $\widehat{A}$.
    \end{algorithmic}
    \label{alg:traj_opt_exploration}
    \end{algorithm}
\end{center}

\begin{thm}
\label{thm:main}
Suppose $\vb x_{t + 1} = \Astar \phi(\vb x_t, \vb u_t) + \vb w_t$ is a nonlinear dynamical system which satisfies Assumptions~\ref{as:lip}-\ref{as:noise} and suppose $\Dcal$ is an initial trajectory that satisfies Assumption~\ref{as:initial}. Also, let $\beta = c_4 \left(d + k \log(\beta_\phi^2 T) + \log(\pi^2 T^2 / (6\delta))\right)^{-1}$ with $c_4 \leq \frac{(c_1 - 2)^2}{36c_3^2}$ and\footnote{Recall $c_1$ is the universal constant appearing in Assumption~\ref{as:controllable} and $c_3$ is the universal constant appearing in the upper bound on the error of the OLS estimate, shown in Section~\ref{sec:ols}.} let
\begin{align}
N_e := \left\lceil \frac{2k \log\left(\frac{2k b_\phi^2}{\log(1 + \beta / 2)}\right)}{\log(1 + \beta / 2)} \right \rceil.
\end{align}
Then, with probability $1 - \delta$, Algorithm~\ref{alg:traj_opt_exploration} with parameters $T$ and $\beta$ outputs $\Ahat$ such that
\begin{align}
\norm{\Ahat - \Astar} \leq c_5 \frac{b_w}{\alpha} \sqrt{\frac{d + k \log(b_\phi^2 T) + \log\left(\frac{\pi^2 T}{6\delta}\right)}{T / H - N_e}}
\end{align}
whenever $T \geq \frac{32 k b_\phi^2 H}{\alpha^2} + HN_e$.  
\end{thm}

There are several aspects of this result worth pointing out. First of all, the statistical rate we obtained in Theorem~\ref{thm:main} has the same form as the standard statistical rate for linear regression, which is $\Ocal\left(b_w \sqrt{\frac{k}{T}} \right)$. 
The two important distinctions are the dependence on the planning horizon $H$ and the controllability term $\alpha$, both of which are to be expected in our case. 
Algorithm~\ref{alg:traj_opt_exploration} uses trajectory planning for data collection and the length of the reference trajectories is at most $H$. Since we can only guarantee one useful measurement per reference trajectory, it is to be expected that we can only guarantee an effective sample size of $T/H$.
The controllability term $\alpha$ is also natural in our result because it quantifies how large the feature vectors can become in different directions. 
Larger feature vectors imply a larger signal-to-noise ratio, which in turn implies faster estimation. 

\subsection{Trajectory planning}
\label{sec:plan}

The trajectory planning routine shown in Algorithm~\ref{alg:traj_opt_exploration} uses the current estimate $\Ahat$ to plan, assuming no process noise, a trajectory from the current state of the system $\vb x_0^R = \vb x_t$ to a high uncertainty region of the feature space, assuming no process noise. More precisely, it finds a sequence of actions $\{\vb u_j^R\}_{j = 0}^r$ which produces a sequence of reference states $\{\vb x_j^R\}_{j = 0}^r$ with the following properties:
\begin{itemize}
\item $\vb x_{j + 1}^R = \Ahat \phi(\vb x_j^R, \vb u_j^R)$,
\item The last reference state-action pair $(\vb x_r^R, \vb u_r^R)$ is either well aligned with $v$, the minimum eigenvector of $\Phi^\top \Phi$, or its feature vector is in a high uncertainty region of the state space. More precisely, $(\vb x_r^R, \vb u_r^R)$ must satisfy one of the following two inequalities: 
\begin{align*}
\left | \langle \phi(\vb x^R_r, \vb u_r^R) , v \rangle \right | \geq \frac{\alpha}{2} \text{ or }  \phi(\vb{x}^R_r, \vb u_r^R)^\top (\Phi^\top \Phi)^{-1} \phi(\vb{x}^R_r, \vb u_r^R) \geq \beta.
\end{align*}
\end{itemize}
It is not immediately obvious that we can always  find such a sequence of inputs. In Section~\ref{sec:proof} we prove that when Assumptions~\ref{as:controllable} and \ref{as:initial} hold the trajectory planning problem is feasible. 

From the study of OLS, discussed in Section~\ref{sec:ols}, we know that the matrix $\Phi^\top \Phi$ determines the uncertainty set of OLS. The larger $\lambda_{\min} \left(\Phi^\top \Phi\right)$ is, the smaller the uncertainty set will be. 
Therefore, to reduce the size of the uncertainty set we want to collect measurements at feature vectors $\phi$ such that the smallest eigenvalues of $\Phi^\top \Phi + \phi \phi^\top$ are larger than the smallest eigenvalues of $\Phi^\top \Phi$. Ideally, $\phi$ is a minimal eigenvector of $\Phi^\top \Phi$. However, we cannot always drive the system to such a feature vector, especially in the presence of process noise. 

Instead, we settle for feature vectors of the following two types. Firstly, trajectory planner tries to drive the system to feature vectors $\phi$ that are well aligned with the minimal eigenvector $v$ of $\Phi^\top \Phi$, i.e. $|\langle \phi, v \rangle| \geq \alpha$. Such a data collection scheme is an instance of E-optimal design \cite{pukelsheim1993optimal}, which has been shown by \citet{wagenmaker2020active} to produce inputs that allow the estimation of linear dynamics at an optimal rate. 

However, if reaching a feature vector that aligns with the minimal eigenvector is not possible, the trajectory planner finds a reference trajectory to a feature vector $\phi$ such that $\phi^\top (\Phi^\top \Phi)^{-1} \phi \geq \beta$. When this inequality holds our uncertainty about the estimate $\Ahat$ in the direction $\phi$ is large. As shown in Section~\ref{sec:proof}, such feature vectors can be encountered only for a small number of iterations.

Finally, trajectory planning is computationally intractable in general. However, in this work we quantify the data requirements of identifying $\Astar$, leaving computational considerations for future work. We assume access to a computational oracle. This assumption is reasonable since trajectory planning is often solved successfully in practice \cite{kavraki1996probabilistic, lavalle2001randomized, zucker2013chomp}. 

\subsection{Trajectory tracking}
\label{sec:track}

Now we detail the trajectory tracking component of our method. We saw that the trajectory planner produces a reference trajectory $\{(\vb x_j^R, \vb u_j^R)\}_{j = 0}^r$, with $r \leq H$. However, the planner assumes no process noise to generate this reference trajectory. Therefore, if we were to simply plug-in the sequence of actions $\{\vb u_j^R\}_{j = 0}^r$ into \eqref{eq:system}, the states of the system would diverge from $\vb x_j^R$. Instead, after observing each state $\vb x_t$ of the system \eqref{eq:system}, our method chooses an input $\vb u_t$  as follows: 

\begin{itemize}
\item Given the current state $\vb x_t$, our method chooses an input $\vb u_t$ such that 
\begin{align*}
\phi(\vb x_t, \vb u_t)^\top (\Phi^\top \Phi)^{-1} \phi(\vb x_t, \vb u_t) \geq \beta
\end{align*} 
if there exists such an input. In other words, if there is an opportunity to greedily collect an informative measurement, our method takes it. If this situation is encountered, the trajectory tracker increments $t$ by $1$ and then stops tracking and returns.

\item If there is no opportunity for greedy exploration, our method chooses an input $\vb u_t$ that minimizes $\norm{\Ahat (\phi(\vb x_t,\vb  u_t) - \phi(\vb x_j^R, \vb u_j^R))}$ and then increments $t$ and $j$ by one ($t$ indexes the time steps of the system \eqref{eq:system} and $j$ indexes the reference trajectory). Therefore, our method uses closed loop control for data generation since minimizing $\norm{\Ahat (\phi(\vb x_t,\vb  u_t) - \phi(\vb x_j^R, \vb u_j^R))}$ requires access to the current state $\vb x_t$. At time $t$ we choose $\vb u_t$ in this fashion in order to minimize the tracking error $\EE \norm{\vb x_{t + 1} - \vb x_{t + 1}^R}^2$ at the next time step, where the expectation is taken with respect to $\vb w_{t}$. 

\item Our method repeats these steps until $j = r$, i.e. until it reaches the end of the reference trajectory. When $j = r$ the trajectory tracker sets $\vb u_t = \vb u_j^R$, increments $t$ by one, and returns. 
\end{itemize}


\section{General Guarantee on Estimation}
\label{sec:ols}

In this section we provide a general upper bound on the error between an OLS estimate $\widehat{A}$ and the true parameters $\Astar$. 
The guarantee is based on the work of \citet{simchowitz2018learning}. However, these types of results have been previously used in the study of online least squares and linear bandits \cite{abbasi2011online, dani2008stochastic,rusmevichientong2010linearly}. 
We assume that we are given a sequence of observations $\{(\vb x_t, \vb u_t, \vb x_{t + 1})\}_{t \geq 0}$ generated by the system \eqref{eq:system}, 
with $u_t$ allowed to depend on $\vb x_0$, $\vb x_1$, \ldots, $\vb x_{t - 1}$ and independent of $\vb w_j$ for all $j \geq t$.   
In what follows we denote $\phi_t := \phi(\vb x_t, \vb u_t)$. 

Our method re-estimates the parameters $\Astar$ as more data is being collected. For the purpose of this section let us denote by $\Ahat_j$ the OLS estimate obtained using the first $j$ measurements $(\vb x_t, \vb u_t, \vb x_{t + 1})$, i.e.
\begin{align}
\label{eq:ols}
\Ahat_j = \arg \min_{A} \sum_{t = 0}^{j - 1} \norm{A \phi_t - \vb x_{t + 1}}^2.
\end{align}

\begin{prop}
\label{prop:ols}
If the dynamical system \eqref{eq:system} satisfies Assumptions~\ref{as:noise} and \ref{as:bounded} and if $\lambda_{\min} \left(\sum_{t = 0}^{t_0 - 1} \phi_t \phi_t^\top \right) \geq \underline{\lambda}$ for some $\underline{\lambda} > 0$ and $t_0 > 0$, the OLS estimates \eqref{eq:ols} satisfy 
\begin{align}
\PP \left[\exists u \in \Sb^{k - 1} \text{ and } j \geq t_0 \text{ s.t. } \norm{(\Ahat_j - \Astar ) u }\geq \mu_j \sqrt{u^\top \left(\sum_{t = 0}^{j - 1} \phi_t \phi_t^\top\right)^{-1} u} \right] \leq \delta,
\end{align}
\end{prop}
where $\mu_j = c_3 b_w \sqrt{d + k \log\left(\frac{b_\phi^2 j}{\underline{\lambda}}\right) + \log\left(\frac{\pi^2 j^2}{6 \delta} \right)}$ for some universal constant $c_3$. 
\begin{proof}
By assumption $\lambda_{\min} \left(\sum_{t = 0}^{t_0 - 1} \phi_t \phi_t^\top \right) \geq \underline{\lambda} > 0$. Therefore, $\sum_{t = 0}^{j - 1} \phi_t \phi^\top_t$ is invertible and 
\begin{align}
\Ahat_j - \Astar = W_j^\top \Phi_j (\Phi_j^\top \Phi_j)^{-1},
\end{align}
where $W_j^\top = [\vb w_0, \ldots, \vb w_{j - 1}]$ and $\Phi_j^\top = [\phi_0, \ldots, \phi_{j - 1}]$. Now, we fix the index $j$ and we consider the SVD decomposition $\Phi_{j} = U \Sigma V^\top$. Therefore, $\Ahat_j - \Astar = W_j^\top U \Sigma^\dagger V^\top$.

Recall that $\sup_{\vb x, \vb u}\|\phi(\vb x, \vb u)\|_2 \leq b_\phi$ by assumption. Then, according to the analysis of \citet{simchowitz2018learning} we know that 
$ \norm{W_j^\top U } \leq \mu_j$ with probability at least $1 - 6\delta / (\pi^2j^2)$. Note that for all $u\in \Sb^{k - 1}$ we have
\begin{align*}
\norm{ (\Ahat_j - A_\star) u } &\leq \norm{W_j^\top U } \norm{\Sigma^\dagger V^\top u} = \norm{W_j^\top U} \sqrt{ u^\top V (\Sigma^\dagger)^\top \Sigma^\dagger V^\top u} \\
&=  \norm{W_j^\top U} \sqrt{u^\top (\Phi_j^\top \Phi_j)^{-1}u}.
\end{align*}
Therefore, for a fixed index $j$, we have 
\begin{align}
\PP \left[\exists u \in \Sb^{k - 1} \text{ s.t. }\|(\Ahat_j - \Astar ) u \|_2 \geq \mu_j \sqrt{u^\top \left(\sum_{t = 0}^{j - 1} \phi_t \phi_t^\top\right)^{-1} u} \right] \leq \frac{6 \delta}{\pi^2 j^2}. 
\end{align}
A direct application of the union bound yields the desired conclusion. 
\end{proof}


\section{Proof of Theorem~\ref{thm:main}}
\label{sec:proof}

First let us observe that when $b_w = 0$ the result is trivial. Because we assume access to an initial trajectory $\Dcal$ which satisfies Assumption~\ref{as:initial} we are guaranteed $\Ahat = \Astar$ when $b_w = 0$. Therefore, we can assume that $b_w > 0$, which implies that $\alpha$ must be strictly positive according to Assumption~\ref{as:controllable}. Throughout the proof we denote $\phi_t  := \phi(\vb x_t, \vb u_t)$ and $\phi_j^R := \phi(\vb x_j^R, \vb u_j^R)$. 

The proof of our result has three parts, which we now outline:
\begin{itemize}
\item We show that the trajectory planning step in Algorithm~\ref{alg:traj_opt_exploration} is always feasible. 
\item We show that during the execution of Algorithm~\ref{alg:traj_opt_exploration} there are at most $N_e$ iterations in which either:
\begin{align}
\max_{\vb u \in \BB_{r_u}} \phi(\vb x_t, \vb u)^\top (\Phi^\top \Phi)^{-1} \phi(\vb x_t, \vb u) \geq \beta \;\text{ or }\; (\phi_j^R)^\top (\Phi^\top \Phi)^{-1} \phi_j^R \geq \beta,
\end{align}
for some $t$ and $j$.

\item We show that Algorithm~\ref{alg:traj_opt_exploration} collects at least $T/H - N_e$ measurements $(\phi_t, \vb x_{t + 1})$ such that $|\langle \phi_t, v\rangle| \geq \alpha/ 4$, where $v$ is a minimal eigenvector used to plan the reference trajectories. As a consequence, we show that Algorithm~\ref{alg:traj_opt_exploration} collects measurements $(\phi_t, \vb x_{t + 1})$ such that 
\begin{align}
\label{eq:min_eig_alg}
\lambda_{\min} \left(\sum_{t = 1}^{T + t_0} \phi_t \phi_t^\top \right) \geq \Ocal(1) \alpha^2 \left(\frac{T}{H} - N_e\right) - \frac{k - 1}{2}b_\phi^2. 
\end{align}
\end{itemize}

Once we have shown \eqref{eq:min_eig_alg} is true, Theorem~\ref{thm:main} follows from Proposition~\ref{prop:ols} and some algebra. 

\subsection*{Part 1 of the proof of Theorem~\ref{thm:main}.} 
We show that the trajectory planning step of Algorithm~\ref{alg:traj_opt_exploration} is always feasible. 
Let 
\begin{align*}
\mu = c_3 b_w \sqrt{d + k \log\left(b_\phi^2 T\right) + \log\left(\frac{\pi^2 T^2}{6 \delta} \right)},
\end{align*} 
where $c_3$ is the universal constant appearing in Proposition~\ref{prop:ols}. Since Assumption~\ref{as:initial} guarantees that the minimum eigenvalue of the design matrix is at least $1$, we know that
\begin{align}
\label{eq:ols_part1}
\norm{(\Ahat - \Astar)\phi} \leq \mu \sqrt{\phi^\top \left(\Phi^\top \Phi\right)^{-1}\phi},
\end{align}
for all $\phi\in \Sb^{k-1}$ and all iterations of Algorithm~\ref{alg:traj_opt_exploration} with probability $1-\delta$. 

Now, let $\beta = \ctbeta \left(d + k \log(\beta_\phi^2 T) + \log(\pi^2 T^2 / (6\delta))\right)^{-1}$ with $\ctbeta \leq \ctcontrol^2 / (4 \ctols^2)$. Then, since $\alpha \geq \ctcontrol L b_w (1 + \gamma + \ldots + \gamma^{H - 1})$, we have 
\begin{align}
\label{eq:beta_bound}
\beta \leq \left(\frac{\alpha}{2 L (1 + \gamma + \ldots + \gamma^{H - 1})\mu}\right)^2. 
\end{align}

Let us $\tilde{\vb x}_0$ be equal to the initial state $\vb x_0^R$ of the trajectory planning and let $v \in \RR^k$ be the desired goal direction. 
By Assumption~\ref{as:controllable} we know that there must exist a sequence of inputs $\tilde{\vb u}_0$, $\tilde{\vb u}_1$, \ldots, $\tilde{\vb u}_r$, with $r \leq H$ and $\norm{\tilde{\vb u}_j} \leq b_u / 2$, such that 
$|\langle \phi(\tilde{\vb x}_r, \tilde{\vb u}_r) , v \rangle | \geq \alpha$, where $\tilde{\vb x}_{j + 1} = \Astar \phi(\tilde{\vb x}_j, \tilde{\vb u}_j)$.  
Now, let $\vb x_{j + 1}^R = \Ahat \phi(\vb x_j^R, \vb u_j^R)$, where $\vb u_j^R$ is any input vector with $\norm{\vb u_j^R} \leq b_u$ such that 
\begin{align}
\label{eq:ur_choice}
\norm{\Astar[\phi(\vb x_j^R, \vb u_j^R) - \phi(\tilde{\vb x}_j, \tilde{\vb u}_j)]} \leq \gamma\norm{\vb x_j^R - \tilde{\vb x}_j} 
\end{align}
for $j < r$.
Assumption~\ref{as:controllable} guarantees the existence of $\vb u_j^R$. We set $\vb u_r^R = \tilde{\vb u}_r$ and denote $\tilde{\phi}_j = \phi(\tilde{\vb x}_j, \tilde{\vb u}_j)$ and $\phi_j^R = \phi(\vb x_j^R, \vb u_j^R)$.

\paragraph{Case 1.} There exists $j \in \{0,1,2, \ldots, r\}$ such that $(\phi_j^R)^\top \left(\Phi^\top \Phi\right)^{-1} \phi_j^R \geq \beta$. 
If this is the case, we are done because we found a feasible sequence of inputs $\vb u^R_0$, $\vb u^R_1$, \ldots, $\vb u^R_j$. 

\paragraph{Case 2.} We have $(\phi_j^R)^\top \left(\Phi^\top \Phi\right)^{-1} \phi_j^R \leq \beta$ for all $j\in \{0,1,2, \ldots, r\}$. In this case, we have 

\begin{align*}
\tilde{\vb x}_{j + 1} - \vb x_{j + 1}^R &= \Astar \tilde{\phi}_j - \Ahat \phi_j^R = \Astar (\tilde{\phi}_j - \phi_j^R) + (\Astar - \Ahat) \phi_j^R. 
\end{align*}
Therefore, using \eqref{eq:ols_part1}, \eqref{eq:beta_bound}, and \eqref{eq:ur_choice} we find
\begin{align*}
\norm{\tilde{\vb x}_{j + 1} - \vb x_{j + 1}^R} &\leq \norm{\Astar (\tilde{\phi}_j - \phi_j^R)} + \norm{(\Astar - \Ahat) \phi_j^R}\\
&\leq \gamma \norm{\tilde{\vb x}_j - \vb x_j^R} + \frac{\alpha}{2 L (1 + \gamma + \ldots + \gamma^{H - 1})}.
\end{align*}
Applying this inequality recursively, we find $\norm{\tilde{\vb x}_{r} - \vb x_{r}^R} \leq \frac{\alpha}{2L}$, which implies 
$|\langle \phi_r^R, v \rangle | \geq \alpha / 2$ because $\norm{\phi_r^R - \tilde{\phi}_r} \leq L \norm{\tilde{\vb x}_{r} - \vb x_{r}^R}$ by Assumption~\ref{as:lip} and $|\langle \tilde{\phi}_r , v \rangle | \geq \alpha$ by construction.
Hence, we constructed a feasible sequence of inputs $\{\vb u_j\}_{j = 0}^r$ and Part 1 of the proof is complete.

\subsection*{Part 2 of the proof of Theorem~\ref{thm:main}.} 

Now, we show that the number of iterations Algorithm~\ref{alg:traj_opt_exploration} encounters 
\begin{align}
\label{eq:fast_shrink}
\max_{\vb u \in \BB_{r_u}} \phi(\vb x_t, \vb u)^\top (\Phi^\top \Phi)^{-1} \phi(\vb x_t, \vb u) \geq \beta \;\text{ or }\; (\phi_j^R)^\top (\Phi^\top \Phi)^{-1} \phi_j^R \geq \beta
\end{align}
is upper bounded by 
\begin{align}
\label{eq:max_num_ellipsoid}
N_e := \left\lceil \frac{2k \log\left(\frac{2k b_\phi^2}{\log(1 + \beta / 2)}\right)}{\log(1 + \beta / 2)} \right \rceil. 
\end{align}  
We rely on the following proposition whose proof is deferred to Appendix~\ref{app:ellipsoid}. 

\begin{restatable}{prop}{ellipsoid}
\label{prop:ellipsoid_update}
Let $M_0$ be a positive definite matrix and let us consider a sequence of vectors $\{v_t\}_{t \geq 1}$ in $\RR^k$ with $\max_{t \geq 1} \norm{v_t} \leq b$. Then, the number of vectors $v_{t + 1}$ such that 
\begin{align*}
v_{t + 1}^\top \left(M_0 + \sum_{i = 1}^t v_i v_i^\top \right)^{-1} v_{t + 1} \geq \beta, 
\end{align*}
is upper bounded by
\begin{align}
\left\lceil \frac{2k \log\left(\frac{2k b^2}{\lambda_k(M_0) \log(1 + \beta)}\right)}{\log(1 + \beta)} \right \rceil. 
\end{align}  
\end{restatable}

According to Proposition~\ref{prop:ellipsoid_update}, to prove \eqref{eq:max_num_ellipsoid} it suffices to show that during each iteration of Algorithm~\ref{alg:traj_opt_exploration} with \eqref{eq:fast_shrink} our method collects a measurement $(\phi_t, \vb x_{t + 1})$ such that $\phi_t^\top (\Phi^\top \Phi)^{-1} \phi_t \geq \beta / 2$. 

By the definition of our trajectory tracker, whenever $\sup_{\vb u} \phi(\vb x_t, \vb u)^\top (\Phi^\top \Phi)^{-1} \phi(\vb x_t, \vb u) \geq \beta$ we collect a measurement $(\phi_t, \vb x_{t + 1})$ such that $\phi_t^\top (\Phi^\top \Phi)^{-1} \phi_t \geq \beta$. 

Next, we show that when $(\phi_j^R)^\top (\Phi^\top \Phi)^{-1} \phi_j^R \geq \beta$, for some $j \leq r$, Algorithm~\ref{alg:traj_opt_exploration} is guaranteed to collect a measurement $(\phi_t, \vb x_{t + 1})$ such that $\phi_t^\top (\Phi^\top \Phi)^{-1} \phi_t \geq \beta / 2$. Let $s$ be the smallest index in the reference trajectory such that $(\phi_s^R)^\top (\Phi^\top \Phi)^{-1} \phi_s^R \geq \beta$. 

For the remainder of this section we re-index the trajectory $\{(\vb x_t, \vb u_t)\}_{t \geq 0}$ collected by Algorithm~\ref{alg:traj_opt_exploration} so that $\vb x_j^R = \vb x_j$ for all $j \in \{0, 1, \ldots, s\}$. Then, we show that $(\phi_s^R)^\top (\Phi^\top \Phi)^{-1} \phi_s^R \geq \beta$ implies the existence of $j \in \{0,1,\ldots, s\}$ such that $\phi_j^\top (\Phi^\top \Phi)^{-1} \phi_j \geq \beta / 2$.

Let $\Delta = \phi_s^R - \phi_s$. The Cauchy-Schwarz inequality implies
\begin{align}
\phi_s (\Phi^\top \Phi)^{-1} \phi_s &= \phi_s^R (\Phi^\top \Phi)^{-1} \phi_s^R + \Delta^\top (\Phi^\top \Phi)^{-1} \Delta + 2\Delta^T (\Phi^\top \Phi)^{-1} \phi_s^R \nonumber \\
&\geq \left(\sqrt{\phi_s^R (\Phi^\top \Phi)^{-1} \phi_s^R} - \sqrt{\Delta^\top (\Phi^\top \Phi)^{-1} \Delta}\right)^2.
\label{eq:delta}
\end{align}
Then, as long as $\Delta^\top (\Phi^\top \Phi)^{-1} \Delta \leq \frac{\beta}{2}(3 - 2\sqrt{2})$, we are guaranteed to have $\phi_s^\top (\Phi^\top \Phi)^{-1} \phi_s \geq \beta / 2$.

Now, since $s$ is the smallest index such that $(\phi_s^R)^\top (\Phi^\top \Phi)^{-1} \phi_s^R \geq \beta$, we know that for all $j \in \{0,1, \ldots, s - 1\}$ we have
$(\phi_j^R)^\top (\Phi^\top \Phi)^{-1}  \phi_j^R \leq \beta$. Also, we can assume that during reference tracking we do not encounter a state $\vb x_j$, with $j \in \{0,1,\ldots, s - 1\}$, such that 
\begin{align*}
\max_{\vb u \in \BB_{r_u}} \phi(\vb x_j, \vb u)^\top (\Phi^\top \Phi)^{-1} \phi(\vb x_t, \vb u) \geq \beta
\end{align*} 
because we already discussed this case. Now, let us consider the difference 
\begin{align*}
\vb x_{j + 1} - \vb x_{j + 1}^R &= \Astar \phi_j + \vb w_j - \Ahat \phi_j^R = (\Astar - \Ahat)\phi_j + \vb w_j - \Ahat [\phi_j^R - \phi_j].
\end{align*}

Therefore, we obtain
\begin{align*}
\norm{\vb x_{j + 1} - \vb x_{j + 1}^R} \leq \norm{(\Astar - \Ahat)\phi_j} + b_w + \norm{\Ahat [\phi_j^R - \phi_j]}. 
\end{align*}
Let us denote $\delta_j(\vb u) = \phi(\vb x_j, \vb u) - \phi_j^R$. Hence, $\delta_j(\vb u_j) = \phi_j - \phi_j^R$. Now, let $\vb u_\star \in \BB_{r_u}$ an input such that $\norm{\Astar \delta_t(\vb u_\star)} \leq \gamma \norm{\vb x_j - \vb x_j^R}$, which we know exists by Assumption~\ref{as:control} (note that $\vb u_\star$ depends on the index $j$, but we dropped this dependency from the notation for simplicity). Since our method attempts trajectory tracking by choosing $\vb u_j \in \arg \min_{\vb u \in \BB_{r_u}} \norm{\Ahat (\phi(\vb x_t,\vb  u) - \phi(\vb x_j^R, \vb u_j^R))}$ we have 
\begin{align*}
\norm{\Ahat \delta_j(\vb u_j)} &\leq \norm{\Ahat \delta_j(\vb u_\star)} \leq \norm{\Astar \delta_j(\vb u_\star)} + \norm{(\Astar - \Ahat) \delta_j(\vb u_\star)} \\
&\leq \gamma \norm{\vb x_j - \vb x_j^R} + \norm{(\Astar - \Ahat) \delta_j(\vb u_\star)}\\
&\leq \gamma \norm{\vb x_j - \vb x_j^R} + \norm{(\Astar - \Ahat)\phi(\vb x_j, \vb u_\star)} + \norm{(\Astar - \Ahat)\phi_j^R}.
\end{align*}
As mentioned above, we can assume $\phi(\vb x_j, \vb u_\star)^\top (\Phi^\top \Phi)^{-1} \phi(\vb x_j, \vb u_\star) < \beta$. Also, recall that 
\begin{align*}
(\phi_j^R)^\top (\Phi^\top \Phi)^{-1}  \phi_j^R \leq \beta
\end{align*} 
since $j < s$ and $s$ is the smallest index so that this inequality does not hold. 
Hence, Proposition~\ref{prop:ols} implies that $\norm{(\Astar - \Ahat)\phi(\vb x_t, \vb u_\star)} \leq \mu \sqrt{\beta}$ and $\norm{(\Astar - \Ahat)\phi_r^R} \leq \mu \sqrt{\beta}$. Putting everything together we find 
\begin{align*}
\norm{\vb x_{j + 1} - \vb x_{j + 1}^R} \leq \gamma \norm{\vb x_j - \vb x_j^R} + 3 \mu \sqrt{\beta} + b_w.
\end{align*}
Then, since the reference trajectory is initialized with the state $\vb x_{0}^R = \vb x_0$, we find 
\begin{align*}
\norm{\Delta} = \norm{\phi_s - \phi_{s}^R} &\leq L(b_w + 3\mu \sqrt{\beta}) (1 + \gamma + \ldots + \gamma^{s - 1})\\
&= (3\ctols \sqrt{\ctbeta} + 1) L b_w (1 + \gamma + \ldots + \gamma^{s - 1}),
\end{align*}
where the last identity follows because $\mu \sqrt{\beta} = \ctols \sqrt{\ctbeta} b_w$. 

Then, as long as $\ctinitial \geq \frac{2(3\ctols \sqrt{\ctbeta} + 1)^2}{(3 - 2\sqrt{2})\ctbeta}$, Assumption~\ref{as:initial} offers a lower bound on $\lambda_{\min} (\Phi^\top \Phi)$ which ensures that $\Delta^\top (\Phi^\top \Phi)^{-1} \Delta \leq \frac{\beta}{2}(3 - 2\sqrt{2})$, implying $\phi_s^\top (\Phi^\top \Phi)^{-1} \phi_s \geq \beta / 2$.

To summarize, we have shown whenever Algorithm~\ref{alg:traj_opt_exploration} encounters a situation in which either 
\begin{align}
\label{eq:fast_shrink_2}
\sup_{\vb u} \phi(\vb x_t, \vb u)^\top (\Phi^\top \Phi)^{-1} \phi(\vb x_t, \vb u) \geq \beta \;\text{ or }\; (\phi_j^R)^\top (\Phi^\top \Phi)^{-1} \phi_j^R \geq \beta,
\end{align}
it collects a measurement $(\phi_t, \vb x_{t + 1})$ such that $\phi_t (\Phi^\top \Phi)^{-1} \phi_t \geq \beta / 2$. Hence, according to Proposition~\ref{prop:ellipsoid_update}, the event \eqref{eq:fast_shrink_2} can occur at most $N_e$ times (the value $N_e$ was defined in \eqref{eq:max_num_ellipsoid}).

\subsection*{Part 3 of the proof of Theorem~\ref{thm:main}.} 

In this final part of the proof we analyze what happens when the trajectory planning problem returns a reference trajectory $(\vb x_j^R, \vb u_j^R)$ for which $|\langle \phi_r^R, v\rangle| \geq \alpha / 2$, where $v$ is a minimal eigenvector with unit norm of $\Phi^\top \Phi$.

We know that there will be at least $T/H$ reference trajectories produced during the run of the algorithm and from Part 2 of the proof we know that at least $T / H - N_e$ of the reference trajectories satisfy $|\langle \phi_r^R, v\rangle| \geq \alpha / 2$, with all states $\vb x_t$ encountered during tracking satisfying $\sup_{\vb u} \phi(\vb x_t, \vb u)^\top \left(\Phi^\top \Phi\right)^{-1} \phi(\vb x_t, \vb u) \leq \beta$ and $(\phi_j^R)^\top (\Phi^\top \Phi)^{-1}  \phi_j^R \leq \beta$ for all $j \in \{0,1,\ldots, r\}$. 
 
Following the same argument as in Part 2 of the proof we know that tracking the reference trajectory in this case takes the system to a state $\vb x_t$ such that 
\begin{align}
\norm{\vb x_t - \vb x_r^R} \leq (3\ctols \sqrt{\ctbeta} + 1) b_w  (1 + \gamma + \ldots + \gamma^{r - 1}),
\end{align}
which implies by Assumption~\ref{as:lip} that 
\begin{align}
\norm{\phi_t - \phi_{r}^R} \leq (3\ctols \sqrt{\ctbeta} + 1) L b_w (1 + \gamma + \ldots + \gamma^{r - 1}). 
\end{align}
This last inequality implies that $|\langle \phi_t, v\rangle| \geq \alpha / 4$ if $3 \ctols \sqrt{\ctbeta} + 1 \leq \ctcontrol / 2$. Recall that the only condition we imposed so far on $\ctbeta$ is $\ctbeta \leq \ctcontrol^2 / (4 \ctols^2)$ in Part 1 of the proof. Hence, since $\ctcontrol > 2$, we can choose $\ctbeta \leq \frac{(\ctcontrol - 2)^2}{36 \ctols^2}$ to ensure that $\ctbeta \leq \ctcontrol^2 / (4 \ctols^2)$ and $3 \ctols \sqrt{\ctbeta} + 1 < \ctcontrol / 2$. 
Now, to finish the proof of Theorem~\ref{thm:main} we rely on the following result, whose proof is deferred to Appendix~\ref{app:rank_one}. 

\begin{restatable}{prop}{rankone}
\label{prop:rank-one}
Let $\Vcal \subset \RR^k$ be a bounded set, with $\sup_{v\in \Vcal}\norm{v} \leq b$, such that for any $u \in \Sb^{k - 1}$ there exists $v\in \Vcal$ with $|\langle u, v \rangle| \geq \alpha$. Then, for all $T \geq 0$, given any sequence of vectors $\{v_t\}_{t \geq 0}$ in $\RR^k$ we have 
\begin{align*}
\lambda_{\min} \left(\sum_{i = 1}^T v_i v_i^\top \right) \geq \frac{\alpha^2 K(T)}{2k} - \frac{k - 1}{2}\left(b^2 - \frac{\alpha^2}{2}\right),
\end{align*}
where $K(T)$ is the number of times 
\begin{align*}
v_{t + 1} \in \left\{v | v\in \Vcal \text{ and } |\langle \tilde{v}_{t + 1}, v \rangle | \geq \alpha \right\}
\end{align*}
with $\tilde{v}_{t + 1} \in \arg \min_{\norm{v} = 1} v^\top \left(\sum_{i = 1}^t v_i v_i^\top\right)v$ and $t < T$.
\end{restatable}

We have shown that at least $T/H - N_e$ times the algorithm collects a state transition $(\vb x_t, \vb u_t, \vb x_{t + 1})$ for which $\phi_t$ is at least $\alpha / 4$ aligned with the minimal eigenvector of $\Phi^\top \Phi$, where $\Phi$ is the matrix of all $\phi_j$ observed prior to the last trajectory planning. Therefore, Proposition~\ref{prop:rank-one} implies that Algorithm~\ref{alg:traj_opt_exploration} collects a sequence of measurements $(\phi_t, \vb x_{t + 1})$ such that 
\begin{align*}
\lambda_{\min} \left(\sum_{t = 1}^{T + t_0} \phi_t \phi_t^\top \right) \geq \frac{\alpha^2}{32} \left(\frac{T}{H} - N_e\right) - \frac{k - 1}{2}b_\phi^2. 
\end{align*}

Putting together this result with Proposition~\ref{prop:ols} yields the desired conclusion.

\section{Related Work}
\label{sec:related}

System identification, being one of the cornerstones of control theory, has a rich history, which we cannot hope to summarize here. For an in-depth presentation of the field we direct the interested reader to the book by \citet{ljung1987system} and the review articles by \citet{aastrom1971system}, \citet{bombois2011optimal}, \citet{chiuso2019system}, \citet{hong2008model}, \citet{juditsky1995nonlinear}, \citet{ljung2020shift}, \citet{schoukens2019nonlinear}, and \citet{sjoberg1995nonlinear}. Instead, we discuss recent studies of system identification that develop finite time statistical guarantees. 

Most recent theoretical guarantees of system identification apply to linear systems under various sets of assumptions \cite{campi2002finite, dahleh1993sample, faradonbeh2018finite, fattahi2019learning, hardt2018gradient,hazan17,hazan18,oymak2019non, sarkar2018near,sarkar2019finite, sarkar2019nonparametric, simchowitz2018learning,simchowitz2019learning,tsiamis2019finite,tsiamis2019sample, wagenmaker2020active}. Notably, \citet{simchowitz2018learning} derived sharp rates for the non-adaptive estimation of marginally stable systems. Then, \citet{sarkar2018near} developed a more general analysis that also applies to a certain class of unstable linear systems. Both of these studies assumed that the estimation method can directly observe the state of the system. We make the same assumption in our work. However, in many applications full state observation is not possible. Recently, 
\citet{simchowitz2019learning} proved that marginally stable linear systems can be estimated from partial observations by using a prefiltered least squares method. 
From the study of linear dynamics, the work of \citet{wagenmaker2020active} is the closest to our own. Inspired by E-optimal design \cite{pukelsheim1993optimal}, the authors propose and analyze an adaptive data collection method for linear system identification which maximizes the minimal eigenvalue $\lambda_{\min}(\sum_{t = 0}^{T - 1} \vb x_t \vb x_t^\top)$ under power constraints on the inputs. \citet{wagenmaker2020active} prove matching upper and lower bounds for their method. 

Comparatively, there is little known about the sample complexity of nonlinear system identification. \citet{oymak2018stochastic} and \citet{bahmani2019convex} studied the estimation of the parameters $A$ and $B$ of a dynamical system of the form $\vb x_{t + 1} = \phi(A \vb x_t + B \vb u_t)$, where $\phi$ is a known activation function and the inputs $u_t$ are i.i.d. standard Gaussian vectors. Importantly, in this model both $\vb x_t$ and $\vb u_t$ are observed and there is no unobserved noise, which makes estimation easy when the map $\phi$ is invertible. In follow-up work, \citet{sattar2020non} and \citet{foster2020learning} generalized these results. In particular, \citet{foster2020learning} took inspiration from the study of generalized linear models and showed that a method developed for the standard i.i.d. setting can estimate dynamical systems of the form $\vb x_{t + 1} = \phi(A \vb x_t) + \vb w_t$ at an an optimal rate, where $\vb w_t$ is i.i.d. unobserved noise. All these works share a common characteristic, they study systems for which identification is possible through the use of non-adaptive inputs. We take the first step towards understanding systems that require adaptive methods for successful identification. 

In a different line of work, \citet{singh2019learning} proposed a learning framework for trajectory planning from learned dynamics. They propose a regularizer of dynamics that promotes stabilizability of the learned model, which allows tracking reference trajectories based on estimated dynamics. Also, \citet{khosravi2020convex} and \citet{khosravi2020nonlinear} developed learning methods that exploit other control-theoretic priors. Nonetheless, none of these works characterize the sample complexity of the problem. 

While most work that studies sample-complexity questions on tabular MDPs focuses on finding optimal policies, \citet{jin2020reward} and \citet{wolfer2018minimax} recently analyzed data collection for MDPs and system identification. More precisely, \citet{jin2020reward} developed an efficient algorithm for the exploration of tabular MDPs that enables near-optimal policy synthesis for an arbitrary number of reward functions, which are unknown during data collection, while \citet{wolfer2018minimax} derived minimax sample complexity guarantees for the estimation of ergodic Markov chains. Finally, we note that \citet{abbeel2005exploration} quantified the sample complexity of learning policies from demonstrations for tabular MDPs and for a simpler version of the model class \eqref{eq:system}. 

\section{Discussion and Open Problems}
\label{sec:takeaways}

System identification led to the development of controllers for many applications and promises to help us tackle many others in the future. In this work we proposed and analyzed a method that estimates a class of nonlinear dynamical systems in finite time by adaptively collecting data that is informative enough. While this results takes us closer to understanding the fundamental limits of data driven control, there are many limitations to our model and approach. We would like to end with a list of open questions:
\begin{itemize}
\item To solve trajectory planning problems we assumed access to a computational oracle. Is it possible to develop a method that has good statistical guarantees and is also computationally tractable? In practice, successful nonlinear control is often based on linearizations of the dynamics. Is it possible to quantify the sample complexity of system identification when trajectory planning is implemented using linearizations? 

\item Our method relies on full state observations. However, in many applications full state observations are impossible. 
Is it possible to gain statistical understanding of nonlinear system identification from partial observations? 

\item Our guarantee holds only when the true system being identified lies in the model class \eqref{eq:system}. When the true system is no part of the model class, how much data is needed to find the best model in class? \citet{ross2012agnostic} studied this problem under a generative model.

\item Only fully actuated systems can satisfy Assumption~\ref{as:control} with $\gamma < 1$. Is it possible to extend our result to systems that require multiple time steps to recover from disturbances?

\item Assumption~\ref{as:controllable} allows only systems whose feature vectors can align with any direction. What if the feature vectors can align only with vectors in a subspace? In this case, it is not possible to recover $\Astar$ fully. However, in this case, it would not be necessary to know $\Astar$ fully in order to predict or control. Is it possible to estimate $\Astar$ only in the relevant directions?

\item What if we consider infinite dimensional feature maps $\phi$? Can we develop a statistical theory of learning RKHS models of dynamical systems? 

\end{itemize}

\subsection*{Acknowledgments}
This research is generously supported in part by ONR awards N00014-17-1-2191, N00014-17-1-2401, and N00014-18-1-2833, NSF CPS award 1931853, and the DARPA Assured Autonomy program (FA8750-18-C-0101).
We would like to thank Alekh Agarwal, Francesco Borrelli, Debadeepta Dey, Munther Dahleh, Ali Jadbabaie, Sham Kakade, Akshay Krishnamurthy, John Langford, and Koushil Sreenath for useful comments and for pointing out relevant related work. We would also like to thank Michael Muehlebach for detailed and valuable feedback on our manuscript. 

\bibliographystyle{abbrvnat}   
\bibliography{nonlin}  

\appendix


\section{Proof of Proposition~\ref{prop:ellipsoid_update}}
\label{app:ellipsoid}

To make this section self-contained we restate Proposition~\ref{prop:ellipsoid_update} here.

\ellipsoid*

The next lemma relates scaling ellipsoids in one direction with the scaling of their volumes. A proof of this result can be found in the work by \citet{abbasi2011online}. 

\begin{lemma}
\label{lem:ellipsoid_volume}
Suppose $M$ and $N$ are two positive definite matrices with $M \succ N \succ 0$. Then, 
\begin{align*}
\sup_{v \neq 0} \frac{v^\top M v}{v^\top N v} \leq \frac{\det(M)}{\det(N)}. 
\end{align*}
\end{lemma}

Now, we are ready to prove Proposition~\ref{prop:ellipsoid_update}. 
We denote by $N_t = N_0 + \sum_{i = 1}^t v_i v_i^\top$. First, we prove that $\det(N_{t + 1}^{-1}) \leq \det(N_t^{-1}) / (1 + \beta)$ whenever $v_{t + 1}^\top N_{t}^{-1} v_{t + 1} \geq \beta$.  

By definition we have $N_{t+1} \succeq N_t \succ 0$. Therefore, $N_{t + 1}^{-1} \preceq N_t^{-1}$. Now, we apply the Sherman-Morrison rank-one update formula to find 
\begin{align}
v_{t + 1}^\top N_{t + 1}^{-1} v_{t + 1} &= v_{t + 1}^\top N_{t}^{-1} v_{t + 1} - \frac{\left(v_{t + 1}^\top N_{t}^{-1} v_{t + 1} \right)^2}{1 + v_{t + 1}^\top N_{t}^{-1} v_{t + 1} } \\
&= \left(1 - \frac{v_{t + 1}^\top N_{t}^{-1} v_{t + 1}}{1 + v_{t + 1}^\top N_{t}^{-1} v_{t + 1} }\right)v_{t + 1}^\top N_{t}^{-1} v_{t + 1}.
\end{align}
Since the function $x \mapsto \frac{x}{1 + x}$ is increasing for $x > -1$, we find
\begin{align}
v_{t + 1}^\top N_{t + 1}^{-1} v_{t + 1} 
&\leq \frac{v_{t + 1}^\top N_{t}^{-1} v_{t + 1}}{1 + \beta}.
\end{align}
whenever $v_{t + 1}^\top N_{t}^{-1} v_{t + 1} \geq \beta$. Then, Lemma~\ref{lem:ellipsoid_volume} implies that $\det(N_{t + 1}^{-1}) \leq \det(N_t^{-1}) / (1 + \beta)$ whenever $v_{t + 1}^\top N_{t}^{-1} v_{t + 1} \geq \beta$, which in turn implies $\det(N_{t + 1}) \geq (1 + \beta) \det(N_t)$ whenever $v_{t + 1}^\top N_{t}^{-1} v_{t + 1} \geq \beta$.

Let us denote by $\lambda_1(t)$, $\lambda_2(t)$, \ldots, $\lambda_k(t)$ the eigenvalues of $N_t$ sorted in decreasing order. Recall that $\lambda_i(t)$ is a non-decreasing function of $t$. 
Now, let $\varepsilon_{i,t} = \log_{1 + \beta}(\lambda_{i}(t) / \lambda_{i}(t - 1))$. Therefore, we have $\lambda_i(t) = (1 + \beta)^{\varepsilon_{i,t}} \lambda_i(t - 1)$. We know $\varepsilon_{i,t} \geq 0$ for all $i$ and $t$ and we know that $\sum_{i = 1}^k \varepsilon_{i,t} \geq 1$ when $v_{t}^\top N_{t - 1}^{-1} v_{t} \geq \beta$ because $\det(N_{t + 1}) \geq (1 + \beta) \det(N_t)$. 

By definition, we have $\lambda_{i}(t) = (1 + \beta)^{\sum_{j = 1}^{t} \varepsilon_{i,j}} \lambda_{i}(N_0) \geq (1 + \beta)^{\sum_{j = 1}^t \varepsilon_{i,j}} \lambda_{k}(N_0)$. Since $\max_j \norm{v_j} \leq b$, we know that $\lambda_i(t + 1) \leq \lambda_i(t) + b^2$. Therefore, 
\begin{align}
(1 + \beta)^{\varepsilon_{i,t + 1}} = \frac{\lambda_{i}(t + 1)}{\lambda_i(t)} \leq 1 + \frac{b^2}{\lambda_{i}(t)} \leq 1 + \frac{b^2}{(1 + \beta)^{\sum_{j = 1}^t \varepsilon_{i,j}} \lambda_{k}(N_0)}.
\end{align}
In other words, we have 
\begin{align}
\varepsilon_{i,t + 1} \leq \frac{\log \left(1 + \frac{b^2}{(1 + \beta)^{\sum_{j = 1}^t \varepsilon_{i,j}} \lambda_{k}(N_0)} \right)}{\log(1 + \beta)} \leq \frac{b^2}{(1 + \beta)^{\sum_{j = 1}^t \varepsilon_{i,j}} \lambda_{k}(N_0) \log(1 + \beta)}.  
\end{align}
Therefore, when $\sum_{j = 1}^t \varepsilon_{i,j} > \log\left(\frac{2k b^2}{\lambda_k(N_0) \log(1 + \beta)} \right) / \log(1 + \beta)$, we have $\varepsilon_{i, t + 1} \leq 1/(2k)$. We denote $\rho = \log\left(\frac{2k b^2}{\lambda_k(N_0) \log(1 + \beta)} \right) / \log(1 + \beta)$. 

Suppose there are $n$ vectors $v_j$ such that $v_{j}^\top N_{j - 1}^{-1} v_{j} \geq \beta$ with $j \leq t$. 
Since $\sum_{i = 1}^k \varepsilon_{i,j} \geq 1$ whenever $v_{j}^\top N_{j - 1}^{-1} v_{j} \geq \beta$, we have $\sum_{j = 1}^t \sum_{i = 1}^k \varepsilon_{i,j} \geq n$. Moreover, at each time $j$ with $v_{j}^\top N_{j - 1}^{-1} v_{j} \geq \beta$ we know that 
\begin{align}
\varepsilon_{i,j} &\geq 1 - \sum_{i^\prime \neq i} \varepsilon_{i^\prime, j} \geq 1 - \sum_{i^\prime : \sum_{s = 1}^{j - 1} \varepsilon_{i^\prime, s} < \rho} \varepsilon_{i^\prime, j} - \sum_{i^\prime : \sum_{s = 1}^{j - 1} \varepsilon_{i^\prime, s} \geq \rho} \varepsilon_{i^\prime, j}\\
&\geq 1 - \sum_{i^\prime : \sum_{s = 1}^{j - 1} \varepsilon_{i^\prime, s} < \rho} \varepsilon_{i^\prime, j} - \sum_{i^\prime : \sum_{s = 1}^{j - 1} \varepsilon_{i^\prime, s} \geq \rho} \frac{1}{2k} \geq \frac{1}{2} - \sum_{i^\prime : \sum_{s = 1}^{j - 1} \varepsilon_{i^\prime, s} < \rho} \varepsilon_{i^\prime, j}. 
\end{align}

Summing these inequalities over $j$, for any $i$ we have 
\begin{align}
\sum_{j = 1}^t \varepsilon_{i,j} &\geq \frac{n}{2} - \sum_{i^\prime \neq i} \min \left\{\log\left(\frac{2k b^2}{\lambda_k(N_0) \log(1 + \beta)}\right) / \log(1 + \beta), \sum_{j = 1}^t \varepsilon_{i^\prime, j} \right\} \\
&\geq \frac{n}{2} - (k - 1) \log\left(\frac{2k b^2}{\lambda_k(N_0) \log(1 + \beta)}\right) / \log(1 + \beta). 
\end{align}
Then, once $n \geq 2k \log\left(\frac{2k b^2}{\lambda_k(N_0) \log(1 + \beta)}\right) / \log(1 + \beta)$, we obtain 
\begin{align*}
\sum_{j = 1}^t \varepsilon_{i,j} \geq \log\left(\frac{2k b^2}{\lambda_k(N_0) \log(1 + \beta)}\right) / \log (1 + \beta),
\end{align*} 
which implies $\varepsilon_{i,j} < \frac{1}{2k}$ for all $j > t$. Since $i$ was chosen arbitrary, we see that whenever $n \geq 2k \log\left(\frac{2k b^2}{\lambda_k(N_0) \log(1 + \beta)}\right)$ and $j > t$ we get $\sum_{i = 1}^k \varepsilon_{ij} < k\frac{1}{2k} < 1$. Hence, $n$ must be smaller or equal than $\left\lceil 2k \log\left(\frac{2k b^2}{\lambda_k(N_0) \log(1 + \beta)}\right) \right\rceil$.

\section{Proof of Proposition~\ref{prop:rank-one}}
\label{app:rank_one}

To make this section self-contained we restate Proposition~\ref{prop:rank-one} here.

\rankone*

\noindent \textbf{Remark:} If we choose $v_i$ to cycle over standard basis vectors, we see that $\lambda_{\min} \left(\sum_{i = 1}^T v_i v_i^\top \right) = \frac{\alpha^2 T}{k}$ when $T$ is a multiple of $k$. Therefore, we cannot hope to have a lower bound in Proposition~\ref{prop:rank-one} better than $\frac{\alpha^2 T}{k}$.
Our proof can be refined to show 
\begin{align*}  
\lambda_{\min} \left(\sum_{i = 1}^T v_i v_i^\top \right) \geq \frac{\alpha^2 T}{k} - \Ocal(\sqrt{T}),
\end{align*}
where the $\Ocal(\cdot)$ notation hides dependencies on $\alpha$, $b$, and $k$.

To prove Proposition~\ref{prop:rank-one} we need the following lemma, which intuitively shows that the sum of the smallest eigenvalues cannot lag behind the larger eigenvalues by too much. 

\begin{lemma}
\label{lem:rank-one}
Let $M \in \RR^{k \times k}$ be a positive semi-definite matrix. Let $\lambda_1 \geq \lambda_2 \geq \ldots \geq \lambda_k$ be the eigenvalues of $M$ and let $u_1$, $u_2$, \ldots, $u_k$ be the corresponding eigenvectors of unit norm. Suppose $v$ is a vector in $\RR^k$ such that $\norm{v} \leq b$ and $|\langle v, u_k\rangle| \geq \alpha$. Let $\nu_1 \geq \nu_2 \geq \ldots \geq \nu_k$ be the eigenvalues of $M + vv^\top$. Then, for any $s \in \{2,\ldots k\}$ such that $\lambda_{s - 1} \geq \lambda_s + b^2 - \alpha^2 / 2$ we have 
\begin{align}
\sum_{i = s}^k \nu_i \geq \sum_{i = s}^k \lambda_i + \frac{\alpha^2}{2}. 
\end{align}
\end{lemma}
\begin{proof}
First, we express $v$ in $M$'s eigenbasis: $v = \sum_{i = 1}^k z_i u_i$. Then, by assumption we know that $\norm{v}^2 = \sum_{i = 1}^k z_i^2 \leq b^2$ and $z_k^2 \geq \alpha^2$. Using a result by Bunch, Nielsen, and Sorensen \cite{bunch978rank} we know that $\nu_1 \geq \lambda_1$ and $\nu_i \in [\lambda_i, \lambda_{i - 1}]$ for every $i \in \{2,\ldots k\}$ and that the $k$ eigenvalues $\nu_i$ are the $k$ solutions of the secular equation:
\begin{align}
f(\nu) := 1 + \sum_{i = 1}^k \frac{z_i^2}{\lambda_i - \nu} = 0
\end{align}
if $z_i \neq 0$ for all $i$. If $z_i = 0$, there is an eigenvalue $\nu_j$ such that $\nu_j = \lambda_i$. We assume $z_i \neq 0$ for all $i$. 

If $\nu_s \geq \lambda_s + \frac{\alpha^2}{2}$, there is nothing to prove. Let us assume $\nu_s <  \lambda_s + \frac{\alpha^2}{2}$. Hence, the eigenvalues $\nu_s$, $\nu_{s + 1}$, \ldots, $\nu_k$ lie in the interval $[\lambda_k, \lambda_s + \alpha^2 / 2)$. For any $\nu \in [\lambda_k, \lambda_s + \alpha^2 / 2)$ we have 

\begin{align}
0 \leq \zeta(\nu) &:= \sum_{i = 1}^{s - 1} \frac{z_i^2}{\lambda_i - \nu} \leq \frac{\sum_{i  = 1}^{s - 1} z_i^2}{\lambda_{s  - 1} - \nu}\\
&\leq \frac{b^2 - \alpha^2}{\lambda_{s  - 1} - \nu} \leq \frac{b^2 - \alpha^2}{b^2 - \alpha^2} = 1. 
\end{align}

By rewriting the equation $f(\nu) = 0$, for any solution $\nu_\star$ which lies in $[\lambda_k, \lambda_s + \alpha^2 / 2)$ we obtain 
\begin{align*}
0 = 1 + \sum_{i = s}^k \frac{z_i^2}{(\lambda_i - \nu_\star)(1 + \zeta(\nu_\star))} \leq 1 + \sum_{i = s}^k \frac{z_i^2}{2(\lambda_i - \nu_\star)} 
\end{align*}
because $1 < 1 + \zeta(\nu_\star) \leq 2$ and $\sum_{i = s}^k \frac{z_i^2}{(\lambda_i - \nu_\star)} < 0$. 
Now, let $\nu_j$ be the unique solution $f(\nu_j) = 0$ in the interval $[\lambda_j, \lambda_{j - 1}]$ for $j \in \{s + 1, \ldots, k\}$ or in the interval $[\lambda_s, \lambda_{s} + \alpha^2 / 2)$ for $j = s$. 

Since the function $g(\nu) = 1 + \sum_{i = s}^k \frac{z_i^2}{2(\lambda_i - \nu)}$ is increasing on the interval $[\lambda_j, \lambda_{j - 1}]$ (if $j = s$, the interval is $[\lambda_s, \infty)$) we know that the unique solution $\nu_j^\prime \in [\lambda_j, \lambda_{j - 1}]$ of the equation $g(\nu) = 0$ satisfies $\nu_j^\prime \leq \nu_j$ for all $j \in \{s, \ldots, k\}$. 

Therefore, we have shown that $\sum_{j = s}^k \nu_j \geq \sum_{j = s}^k \nu_j^{\prime}$, where $\nu_j^\prime$ are the solutions to the equation 
\begin{align*}
1 + \sum_{i = s}^k \frac{z_i^2}{2(\lambda_i - \nu)} = 0.
\end{align*}
However, the solutions of this equation are the eigenvalues of the matrix $Q = \diag(\lambda_s, \lambda_{s + 2}, \ldots, \lambda_k) + \frac{1}{2} zz^\top$, where $z = [z_s, z_{s + 1}, \ldots, z_k]^\top$. Hence, 
\begin{align*}
\sum_{j = s}^k \nu_j &\geq \sum_{j = s}^k \nu_j^\prime = \tr(Q) = \sum_{j = s}^k \lambda_j + \frac{1}{2} \sum_{j = s}^k z_j^2 \geq \sum_{j = s}^k \lambda_j + \frac{\alpha^2}{2}. 
\end{align*}
\end{proof}

Now we can turn back to the proof of Proposition~\ref{prop:rank-one}. Let $\lambda_i(t)$ be the $i$-th largest eigenvalue of $\sum_{j = 1}^t v_j v_j^\top$ and let $K(t)$ be the number of times \begin{align*}
v_{j + 1} \in \left\{v | v\in \Vcal \text{ and } |\langle \tilde{v}_{j + 1}, v \rangle | \geq \alpha \right\}
\end{align*}
with $\tilde{v}_{j + 1} \in \arg \min_{\norm{v} = 1} v^\top \left(\sum_{i = 1}^j v_i v_i^\top\right)v$ and $j < t$.

Suppose we know that $\sum_{i = j - 1}^{k} \lambda_i(t) \geq c_{j - 1} \alpha^2 K(t) - d_{j - 1}$ for all $t \geq 1$, where $c_{j - 1} > 0$ and $d_{j - 1} \geq 0$ are some real values. Since $\norm{v_j} \geq \alpha$ for all $j$, we can choose $c_1 = 1$ and $d_1 = 0$. Now, we lower bound $\sum_{i = j}^k \lambda_i(t)$ as a function of $t$. To this end, we define 
$t_j$ to be the maximum time in $\{1,2,\ldots, t\}$ such that $\lambda_{s - 1}(t_2) - \lambda_s(t_2) < b^2 - \frac{\alpha^2}{2}$ for all $s \in \{j, j + 1,\ldots, k\}$. 

Then, Lemma~\ref{lem:rank-one} and our induction hypothesis guarantee that 
\begin{align*}
\sum_{i = j}^k \lambda_i(t) \geq \sum_{i = j}^k \lambda_i(t_j) + \frac{\alpha^2 (K(t)- K(t_j))}{2} \geq& \frac{\alpha^2(K(t) - K(t_j))}{2} + c_{j - 1} \alpha^2 K(t_j) - d_{j - 1} - \lambda_{j - 1}(t_j). 
\end{align*}
By the definition of $t_j$ we know that $\lambda_i(t_j) \geq \lambda_{j - 1}(t_j) - (i - j + 1)(b^2 - \alpha^2)$ for all $i \geq j$. Therefore, we have the lower bound:
\begin{align*}
\sum_{i = j}^k \lambda_i(t) \geq \sum_{i = j}^k \lambda_i(t_j) \geq (k - j + 1)\lambda_{j - 1}(t_j) - \frac{(k - j + 1)(k - j + 2)}{2} \left(b^2 - \frac{\alpha^2}{2}\right). 
\end{align*}
We minimize the maximum of the previous two lower bounds with respect to $\lambda_{j - 1}(t_j)$, which can be done by finding the value of $\lambda_{j - 1}(t_j)$ which makes the two lower bounds equal. Then, we find
\begin{align*}
\sum_{i = j}^k \lambda_i(t) \geq \frac{\alpha^2}{2} \frac{k - j + 1}{k - j + 2} \left((2 c_{j - 1} - 1) K(t_j) + K(t) \right) - \frac{k - j + 1}{k - j + 2} d_{j - 1} - \frac{k - j + 1}{2} \left(b^2 - \frac{\alpha^2}{2}\right). 
\end{align*}

\textbf{Case 1:} $2c_{j - 1} \geq 1$. Then, since $K(t_j) \geq 0$, we obtain 
\begin{align*}
\sum_{i = j}^k \lambda_i(t) \geq \frac{\alpha^2}{2} \frac{k - j + 1}{k - j + 2} K(t) - \frac{k - j + 1}{k - j + 2} d_{j - 1} - \frac{k - j + 1}{2} \left(b^2 - \frac{\alpha^2}{2}\right). 
\end{align*}

\textbf{Case 2:} $2c_{j - 1} < 1$. Then, since $K(t_j) \leq K(t)$, we obtain 
\begin{align*}
\sum_{i = j}^k \lambda_i(t) \geq \alpha^2 \frac{k - j + 1}{k - j + 2} c_{j - 1} K(t) - \frac{k - j + 1}{k - j + 2} d_{j - 1} - \frac{k - j + 1}{2} \left(b^2 - \frac{\alpha^2}{2}\right). 
\end{align*}

We see that $c_2 = \frac{1}{2} \frac{k - 1}{k} < \frac{1}{2}$ and $\frac{k - j + 1}{k - j + 2} c_{j - 1} < c_{j - 1}$. Therefore, the following recursions hold
\begin{align*}
c_{j} &= \frac{k - j + 1}{k - j + 2}c_{j - 1}\\
d_j &= \frac{k - j + 1}{k - j + 2} d_{j - 1} + \frac{k - j + 1}{2}\left(b^2 - \frac{\alpha^2}{2}\right), 
\end{align*}
with $c_2 = \frac{k - 1}{2k}$ and $d_2 = \frac{k - 1}{2}\left(b^2 - \frac{\alpha^2}{2}\right)$. By unrolling the recursions, we obtain the conclusion.

\section{Refinement of Assumption~\ref{as:initial}}
\label{app:refine}

We saw that when Assumption~\ref{as:initial} offers a lower bound 
\begin{align}
\lambda_{\min} \left(\sum_{t = 0}^{t_0 - 1} \phi(\vb x_t, \vb u_t) \phi(\vb x_t, \vb u_t)^\top \right) \geq 1 + \ctinitial b_w^2 L^2 \left(\sum_{i = 0}^{H - 1} \gamma^i\right) \left(d + k \log(b_\phi^2 T) + \log\left(\frac{\pi^2 T^2 }{6\delta}\right)\right)
\end{align}
Algorithm~\ref{alg:traj_opt_exploration} with parameter $\beta = \ctbeta \left(d + k \log(\beta_\phi^2 T) + \log(\pi^2 T^2 / (6\delta))\right)^{-1}$ is guaranteed to collect measurements $(\phi_t, \vb x_{t + 1})$ such that 
\begin{align}
\label{eq:part3_guarantee}
\lambda_{\min} \left(\sum_{t = 0}^{T + t_0} \phi_t \phi_t^\top \right) \geq \frac{\alpha^2}{32} \left(\frac{T}{H} - N_e(T)\right) - \frac{k - 1}{2}b_\phi^2. 
\end{align}

We wrote $N_e(T)$ because $N_e$ is a function of $\beta$ and $\beta$ is a function of $T$. 
Then, let us consider $T_\star$ to be the smallest value such that 
\begin{align}
\label{eq:tstar}
\frac{\alpha^2}{32} \left(\frac{T_\star}{H} - N_e(T_\star)\right) - \frac{k - 1}{2}b_\phi^2 \geq 1 + \ctinitial b_w^2 L^2 \left(\sum_{i = 0}^{H - 1} \gamma^i\right) \left(d + k \log(2b_\phi^2 T_\star) + \log\left(\frac{\pi^2 T_\star^2 }{3\delta}\right)\right).
\end{align}
Such $T_\star$ exists because $N_e(T_\star) = \Ocal(\log(T_\star))$. Moreover, if \eqref{eq:tstar} holds, any $T \geq T_\star$ satisfies \eqref{eq:tstar}. 

Now, suppose Assumption~\ref{as:initial} is satisfied with $T$ replaced by $T_\star$. Then, we can set $\beta = \ctbeta \left(d + k \log(\beta_\phi^2 T_\star) + \log(\pi^2 T_\star^2 / (6\delta))\right)^{-1}$ for all iterations of Algorithm~\ref{alg:traj_opt_exploration} while $t \leq T_\star$. When $t > T_\star$ inequalities \eqref{eq:part3_guarantee} and \eqref{eq:tstar} guarantee that we can update $\beta$ to be equal to 
\begin{align*}
\ctbeta \left(d + k \log(2\beta_\phi^2 T_\star) + \log(2\pi^2 T_\star^2 / (3\delta))\right)^{-1}
\end{align*} 
and that the minimal eigenvalue of $\sum_{t = 0}^{T} \phi_t \phi_t^\top$ is sufficiently large as long as $T\leq 2 T_\star$. 
Therefore, we can update $\beta$ in epochs of doubling lengths and have Theorem~\ref{thm:main} hold as long as Assumption~\ref{as:initial} holds with $T$ replaced by $T_\star$. 

\newpage

\section{Detailed pseudo-code of Algorithm~\ref{alg:traj_opt_exploration}}
\label{app:alg}

\begin{center}
  \begin{algorithm}[h!]
    \caption{Active learning for nonlinear system identification}{}
    \begin{algorithmic}[1]
      \REQUIRE{Parameters: the feature map $\phi$, initial trajectory $\Dcal$, and parameters $T$, $\alpha$, and $\beta$.}
      \STATE Initialize $\Phi$ to have rows $\phi(\vb x_j, \vb u_j)^\top$ and $Y$ to have rows $(\vb x_{j + 1})^\top$, for $(\vb x_j, \vb u_j, \vb x_{j + 1}) \in \Dcal$.  
      \STATE Set $\widehat{A}  \gets Y^\top \Phi (\Phi^\top \Phi)^{-1}$, i.e. the OLS estimate according to $\Dcal$.  
      \STATE Set $t \gets t_0$.
      \WHILE{$t \leq T + t_0$}
      \STATE Set $\vb x_0^R \gets \vb x_t$,
      \STATE Set $v$ to be a minimal eigenvector of $\Phi^\top \Phi$, with $\norm{v} = 1$. 
      \STATE \textbf{Trajectory planning:} find inputs $\vb u_0^R$, $\vb u_1^R$, \ldots, $\vb u_r^R$, with $\norm{\vb u_j^R} \leq b_u$ and $r \leq H$, such that 
      \begin{align*}
&\left | \langle \phi(\vb x^R_r, \vb u_r^R) , v \rangle \right | \geq \frac{\alpha}{2} \text{ or }  \phi(\vb{x}^R_r, \vb u_r^R)^\top (\Phi^\top \Phi)^{-1} \phi(\vb{x}^R_r, \vb u_r^R) \geq \beta,
     \end{align*}
     where $\vb x_{j + 1}^R = \Ahat \phi(\vb x_j^R, \vb u_j^R)$ for all $j \in \{0, 1, \ldots, r - 1\}$.  
     \label{line:planning}

     \STATE \textbf{Trajectory tracking:}
      \FOR{$j = 0, \ldots, r$}
        \IF{$\max_{\vb u \in \BB_{r_u}} \phi(\vb x_t, \vb u)^\top (\Phi^\top \Phi)^{-1} \phi(\vb x_t, \vb u) \geq \beta$}
          \STATE Set $\vb u_t \in \arg\max_{\vb u \in \BB_{r_u}} \phi(\vb x_t, \vb u)^\top (\Phi^\top \Phi)^{-1} \phi(\vb x_t, \vb u)$,
          \label{line:greedy}
          \STATE Input $\vb u_t$ into the real system and observe the next state: $\vb x_{t + 1} = A_\star \phi(\vb x_t, \vb u_t) + \vb w_t$,
          \STATE $t \leftarrow t + 1$,
          \STATE \textbf{break.}
        \ELSIF{$j \leq r - 1$}
          \STATE Set $\vb u_t \in \arg \min_{\vb u \in \BB_{r_u}} \norm{\Ahat (\phi(\vb x_t,\vb  u) - \phi(\vb x_j^R, \vb u_j^R))}$.
          \label{line:control}
        \ELSE
          \STATE $\vb u_t = \vb u_r^R$. 
        \ENDIF
      \STATE Input $\vb u_t$ into the real system and observe the next state: $\vb x_{t + 1} = A_\star \phi(\vb x_t, \vb u_t) + \vb w_t$,
      \STATE $t \leftarrow t + 1$.
      \ENDFOR
      \STATE  Set $\Phi^\top \gets [\phi_0, \phi_1, \ldots, \phi_{t -1}]$ and $Y^\top \gets [\vb x_1, \vb x_2, \ldots, \vb x_{t}]$, where $(\phi_j, \vb x_{j + 1})$ are all feature-state transitions observed so far. 
      \STATE \textbf{Re-estimate:} $\widehat{A} \gets Y^\top \Phi (\Phi^\top \Phi)^{-1}$.
      \ENDWHILE
    \STATE Output the last estimates $\widehat{A}$.
    \end{algorithmic}
    \label{alg:traj_opt_exploration_detailed}
    \end{algorithm}
\end{center}

\end{document}